\documentclass{article}

\usepackage{arxiv}

\usepackage[utf8]{inputenc} 
\usepackage[T1]{fontenc}    
\usepackage{hyperref}       
\usepackage{url}            
\usepackage{booktabs}       
\usepackage{amsfonts}       
\usepackage{nicefrac}       
\usepackage{microtype}      
\usepackage{comment}
\usepackage{amsthm}
\usepackage{amsmath}
\usepackage{amssymb}
\usepackage{natbib}
\usepackage{color}

\newtheorem{theorem}{Theorem}  

\title{A Note on High-Probability versus In-Expectation Guarantees of Generalization Bounds in Machine Learning}

\author{
  Alexander Mey \\
  Delft University of Technology\\
  \texttt{a.mey@tudelft.nl} \\

}

\begin{document}
\maketitle

\begin{abstract}
Statistical machine learning theory often tries to give generalization guarantees of machine learning models. Those models naturally underlie some fluctuation, as they are based on a data sample. If we were unlucky, and gathered a sample that is not representative of the underlying distribution, one cannot expect to construct a reliable machine learning model. Following that, statements made about the performance of machine learning models have to take the sampling process into account. The two common approaches for that are to generate statements that hold either in high-probability, or in-expectation, over the random sampling process. In this short note we show how one may transform one statement to another. As a technical novelty we address the case of unbounded loss function with the so called the witness condition.
\end{abstract}


\section{Introduction}
In statistical learning theory one often tries to provide generalization guarantees for machine learning models. Machine learning models are based on training data, and thus vary naturally based on the quality of this data. Providing performance guarantees we thus have to take the sampling process of the data into account, and can only generate theoretical guarantees with \emph{high-probability} or \emph{in-expectation} over the training data generated in the sampling process. In statistical learning theory one finds more often high-probability results, as also formalized through the PAC-learning framework \cite[Chapter 3]{understanding}, but not exclusively. 

Regarding this, one might ask the question how to translate one type of guarantee to the another. Usually one cannot assume that there is a lossless transformation between those types of results, and we have to derive both individually to gain optimal guarantees. \cite{svmexpectation} for example showed already early on that a hard-margin support vector machine has an expected error bounded by $O\left(\mathbb{E}[\min\{d+1, \frac{1}{\eta^2_{n+1}} \}] \frac{1}{n+1}\right)$. Here $d$ is the dimension of the input space, $\eta$ is the resulting margin of the support vector machine (and thus a random variable over which we take an expectation), and $n$ is the sample size. Most importantly we observe that the expected error drops as $\frac{1}{n}$, while this convergence rate was only very recently shown in a high-probability statement \citep{bousquet}. 

Although there is no general lossless transformation between high-probability and in-expectation statements,
 we nevertheless want to show how \emph{a} transformation is still possible. While this in itself may be nothing new for versed researchers in statistical learning, we additionally address the case of unbounded loss functions, which is less well known. For this case one needs additional assumptions to be able to translate the high-probability result to an in-expectation result, and in this paper we use the by \cite{witnesspaper} recently introduced witness condition for that.

\section{Preliminaries}
We assume that there is an unknown probability distribution $\mathbb{P}$ on a joint space $\mathcal{X} \times \mathcal{Y}$. Here $\mathcal{X}$ should be thought of the feature, or covariate, of an object, based on which we are supposed to predict a response variable $\mathcal{Y}$. Given a set of hypothesis predictors $\mathcal{F}$, with $f:\mathcal{X} \to \mathcal{Y}$ for each $f \in \mathcal{F}$, we try to find an $f \in \mathcal{F}$ with a small risk $R(f):=\mathbb{E}_{x,y} \left[ l(f(x),y) \right]$. Here $l:\mathcal{Y} \times \mathcal{Y} \to \mathbb{R}$ is a loss function that measures the performance of the predictor $f$. We usually find a specific predictor, which we will denote by $f_n$, with the help of an $n$-sample $(x_i,y_i)_{1\leq i \leq n}$, where each $(x_i,y_i)$ is an i.i.d. draw from $\mathbb{P}$.  Furthermore let $f^*:=\arg\min\limits_{ f \in \mathcal{F}}R(f)$. We define $X_f:=R(f)-R(f^*)$ as the \emph{excess risk} of $f$ and set $L_f(x,y):=l(f(x),y)-l(f^*(x),y)$ as the point-wise excess loss.
\subsection{High-Probability and In-Expectation}
We now define what we mean by a high-probability and an in-expectation statement. As before let $f_n \in \mathcal{F}$ be a predictor based on an $n$-sample. Note that $X_{f_n}$ is then a random variable over the sampling process, and $\mathbb{E}_n [X_{f_n}]$ denotes the expectation of $X_{f_n}$ over the randomly drawn $n$-sample. We say that $X_{f_n}$ is with probability of at least $(1-\delta)$ smaller than $\epsilon(\delta,n)$, for $\epsilon:[0,1] \times \mathbb{N} \to \mathbb{R}$, if for all $\delta >0$ and $n \in \mathbb{N}$
\begin{equation} \label{eq:highprobability}
P(X_{f_n}>\epsilon(\delta,n))<\delta.
\end{equation}
In statistical learning theory one would often come across the equivalent statement, that $X_{f_n}=R(f_n)-R(f^*) \leq \epsilon(\delta,n)$ with probability of at least $1-\delta$.

Similarly we say that
$X_{f_n}<\gamma(n)$ holds in-expectation if
\begin{equation} \label{eq:inexpectation}
\mathbb{E}_n[X_{f_n}] < \gamma(n).
\end{equation}
The question we want to answer in this short note is how we can relate $\delta,\epsilon(\delta,n)$ to $\gamma(n)$. In particular we are interested in how well we can preserve the dependence on the sample size $n$.
\section{High-Probability Implies In-Expectation}
Obtaining an in-expectation result from a high-probability statement is problematic if we might choose with a very small probability a very bad predictor $f_n$. If the loss function is unbounded, so might be then the excess risk $X_{f_n}$, and we cannot expect to find a general way to transform a high-probability result into in-expectation result: The excess risk could be $0$ with very high probability, but arbitrarily large otherwise, making the expected excess risk arbitrarily large as well. The \emph{witness condition} rules this behaviour out.

\paragraph{The Witness Condition}
The intuition on the witness condition given by \cite{witnesspaper} is: `the witness condition says that whenever $f \in \mathcal{F}$ is worse than $f^*$ in-expectation, the probability that we \emph{witness} this in our training example should not be negligibly small'. The formal definition is the following. A collection of random variables $L_f$, $f \in \mathcal{F}$, fulfills the $(u,c)$-witness condition, with $u>0$ and $c \in (0,1]$, if for all $f \in \mathcal{F}$
\begin{equation} \label{eq:witness}
\mathbb{E}[L_f \cdot I(L_f \leq u)] \geq c \mathbb{E} [L_f],
\end{equation}
where $I(A)=1$ if $A$ is true and $I(A)=0$ otherwise, and $
\mathbb{E}[L_f]$ is short for $\mathbb{E}_{x,y}[L_f(x,y)]$.\footnote{\cite{witnesspaper} also have a generalization of the witness condition, where $u$ is replaced with a function $\tau(\mathbb{E} L_f)$. With this more general definition our proof of Theorem \ref{thm:witness} would actually not work.} For our purpose we think of $L_f$ as the previously defined point-wise excess loss.
The witness condition is actually fairly weak, as it can for example still hold for heavy-tailed distributions. See \cite{witnesspaper} for more intuition and details, in particular Section 5.2 gives examples in which the witness condition holds.

The following theorem describes a possible transformation from a high-probability statement to an in-expectation statement under the witness condition. For a technical reason we first define with abuse of notation $u:=\max \{ u,\epsilon(\delta,n)\}$. This is fine as when the $(u,c)$-witness condition holds, so does then the $(\max \{ u,\epsilon(\delta,n)\},c)$-witness condition, as the max operator increases the left hand side of Inequality \eqref{eq:witness}.
\begin{theorem} \label{thm:witness}
Assume that $X_{f_n}$ is with probability of at least $1-\delta$ smaller than $\epsilon(\delta,n)$. If $\{L_{f}\}_{f \in \mathcal{F}}$ fulfills the $(u,c)$ witness condition, then 
\begin{equation} \label{expectation}
\mathbb{E}_n[X_{f_n}] \leq \frac{1}{c}\left[ \epsilon(\delta,n)+(u-\epsilon(\delta,n))\delta \right].
\end{equation}
Note that this also implies the case of bounded loss functions. In this case, there exists a $B \in \mathbb{R}$, such that $X_f \leq B$ for all $f \in \mathcal{F}$. We then have a witness condition with $u = B$ and $c=1$.
\end{theorem}
\begin{proof}
For convenience we will ignore the dependence on $n$ and $\delta$ for $\epsilon(\delta,n)$ and just write $\epsilon$. 
\begin{align} \label{zero}
& \mathbb{E}_n[X_{f_n}]=\mathbb{E}_n [\mathbb{E}_{x,y} [L_{f_n}]] \\ \label{one}
& \leq \frac{1}{c} \left( \mathbb{E}_n \left[\mathbb{E}_{x,y} [L_{f_n} \cdot I(L_{f_n} \leq u)]\right]   \right) \\ \label{two}
 & =\frac{1}{c} \left( \mathbb{E}_n [\mathbb{E}_{x,y}[ L_{f_n} \cdot I(L_{f_n} \leq u) ]\cdot(I(X_{f_n}> \epsilon)+I(X_{f_n} \leq \epsilon))]\right) \\ \label{three}
 &= \frac{1}{c} \left( \mathbb{E}_n [\mathbb{E}_{x,y}[ L_{f_n} \cdot I(L_{f_n} \leq u)] \cdot I(X_{f_n}> \epsilon)]+ \mathbb{E}_n [\mathbb{E}_{x,y} [L_{f_n} \cdot I(L_{f_n} \leq u) ]\cdot I(X_{f_n} \leq \epsilon)]\right) \\ \label{four}
 & \leq \frac{1}{c} \left( \mathbb{E}_n [u \cdot I(X_{f_n}> \epsilon)]+ \mathbb{E}_n [X_{f_n} \cdot I(X_{f_n} \leq \epsilon)]\right) \\ \label{five}
 & \leq \frac{1}{c}(u P(X_{f_n} > \epsilon)+\epsilon(1- P(X_{f_n} > \epsilon)) ) \\ \label{six}
 &=\frac{1}{c}((u-\epsilon) P(X_{f_n} > \epsilon)+\epsilon) \leq  \frac{1}{c}((u-\epsilon) \delta +\epsilon)
\end{align}
Here \eqref{zero}-\eqref{one} follows from the witness condition applied to $L_{f_n}$, and the fact that the expectation is a monotonic operator. Step \eqref{three}-\eqref{four} follows by noting that $\mathbb{E}_{x,y} L_{f_n} \cdot I(L_{f_n} \leq u) \leq u$ as well as that $X_{f_n}=\mathbb{E}_{x,y} L_{f_n}$ together with $I(\cdot) \leq 1$. The last steps follow from rearranging and our assumption $P(X_{f_n} >\epsilon) \leq \delta$, and that $u-\epsilon$ is always non-negative with our re-definition of $u=\max\{u,\epsilon(\delta,n)\}$.
\end{proof}
If we converge with high-probability, then for $n \rightarrow \infty$ we have that for all $\delta>0$ it holds that $\epsilon(\delta,n) \rightarrow 0$. Letting $\delta \rightarrow 0$ at an appropriate speed we see that the right hand side of \eqref{expectation} converges to 0.
Note, however, that $\delta$ is a free variable in Inequality \eqref{expectation}, so we can chose to obtain an optimal bound.
We are not aware of any closed form solution for the $\delta$ that minimizes this bound. For specific examples, however, we might be able to compute the exact rates.
\paragraph{Simple Example}
For $\epsilon(\delta,n)=\frac{1}{\delta n}$ one can show that $\delta=\sqrt{\frac{1}{u n}}$ optimizes the bound. It can be then computed that
\begin{equation}
\mathbb{E}[X_{f_n}] \leq \frac{1}{c}(2\sqrt{\frac{u}{n}}-\frac{1}{n}).
\end{equation}
We observe that the linear rate of $\frac{1}{n}$ for the high-probability rate drops to a $\frac{1}{\sqrt{n}}$ rate for the expectation case, so this example shows that we cannot expect to maintain the same rate after the transformation. It is unclear to us if that is an artefact of the proof technique, necessary because of the linear $\frac{1}{\delta}$ or if we can only preserve the $\frac{1}{n}$ rate under stronger additional assumptions.

\paragraph{Statistical Learning Example}
In statistical learning theory one often observes that $\epsilon(\delta,n)=\ln(\frac{1}{\delta})\frac{1}{n}$ \citep[Theorem 2.5]{foundations}. Setting $\delta=\frac{1}{n}$ we then observe that $\mathbb{E}[X_{f_n}]\leq \frac{1}{c}(\frac{\ln(n)}{n}+(u-\frac{\ln(n)}{n})\frac{1}{n})$. So in this setting we lose at most a factor of $\ln(n)$ when we transform a high-probability to an in-expectation statement.

\section{In-Expectation Implies With High-Probability}
That an in-expectation result implies a high-probability result follows directly from Markov's Inequality, and the fact that the excess risk is always non-negative.
\begin{theorem}
Let $X_{f_n}$ be a non-negative random variable, as for example the excess risk of a trained model $f_n$, and let $\epsilon>0$. If $\mathbb{E}_n [X_{f_n}] \leq \gamma(n)$, then
\begin{equation}
    P(X_{f_n} > \epsilon) \leq \frac{\gamma(n)}{\epsilon}. 
\end{equation}
\end{theorem}

Setting now $\frac{\gamma(n)}{\epsilon}=\delta$ and solving for $\epsilon$ we get that $X_{f_n}$ is with probability of at least $1-\delta$ smaller than $\epsilon(\delta,n)=\frac{\gamma(n)}{\delta}$. From statistical learning theory we would not expect that this is ideal, as the dependence in $\delta$ is usually given as a multiplicative factor of $\ln \frac{1}{\delta}$. 

We could, however, follow a slightly different strategy. Using a more general version of Markov's Inequality, we have the following theorem.

\begin{theorem}
Let $X_{f_n}$ be a non-negative random variable, as for example the excess risk of a trained model $f_n$, and let $\epsilon>0$. Furthermore let $\phi:\mathbb{R}_+ \to \mathbb{R_+}$ be a monotonically increasing non-negative function. Then \begin{equation}
    P(X_{f_n} > \epsilon) \leq \frac{\mathbb{E}_n [\phi(X_{f_n})]}{\phi(\epsilon)}.
\end{equation}
\end{theorem}

The case of $\phi$ being concave seems at first appealing, as Jensen's Inequality allows us then to further bound $\mathbb{E}_n[ \phi (X_{f_n}) ]\leq \phi(\gamma(n))$. It turns, however, out that for concave functions the identity is already ideal. More interesting cases may arise for convex functions, as for example $\phi(\lambda):=e^{\lambda}$. In this case we need to find directly bounds on $\mathbb{E}_n[ \phi(X_{f_n})]$ as Jensen cannot help in this case. If we would find a bound $\mathbb{E} \phi(X_{f_n}) \leq \beta(n)$, we can solve $\epsilon=\phi^{-1}(\frac{\beta(n)}{\delta})$. If $\phi$ is the exponential function for example, we can then solve $\epsilon=\ln(\beta(n))+\ln(\frac{1}{n})$. In this case we indeed have logarithmic dependence on $\delta$, although additive, and not multiplicative, as we would more often encounter in statistical learning.

More generally one could try to find assumptions under which we can also upper bound $\mathbb{E}_n [\phi X_{f_n}]$ by $\mathbb{E}_n[ X_{f_n}]$ for convex $\phi$. In this case we could again find high-probability bounds using the original in-expectation bound $\gamma(n)$. \citet[Lemma 13]{witnesspaper} actually do this, but reversed in the sense that they find a term that upper bounds $\mathbb{E}_n[ X_{f_n}]$ with the help of $\mathbb{E}_n[ \phi(X_{f_n})]$ for the \emph{concave} function $\phi(\lambda)=e^{-\lambda}$. Maybe similar assumptions would also help in the convex case. 

 \section*{Acknowledgement}
I greatly appreciate the comments of Marco Loog, Tom Viering and Christina G{\"o}pfert, as they pointed out technical and non-technical mistakes that happen even in the shortest of documents.

\bibliographystyle{plainnat}
\bibliography{references.bib}

\end{document}